\documentclass{article}

\usepackage{microtype}
\usepackage{graphicx}
\usepackage{subfigure}
\usepackage{booktabs} 

\usepackage{hyperref}

\usepackage[accepted]{icml2025}

\usepackage{amsmath}
\usepackage{amssymb}
\usepackage{mathtools}
\usepackage{amsthm}
\usepackage{wrapfig}
\usepackage{caption}
\usepackage{thm-restate}
\usepackage{amsmath}  
\usepackage{amssymb}  
\usepackage{amsfonts} 
\usepackage{thmtools}
\usepackage{subfigure}
\usepackage{caption}
\usepackage{amssymb}
\usepackage{algorithm}
\usepackage{algorithmic}
\usepackage{graphicx}
\usepackage{thm-restate}
\usepackage{caption}
\usepackage{booktabs}
\usepackage{wrapfig}
\usepackage{multirow}
\usepackage{hyperref}       
\usepackage{url}            
\usepackage{booktabs}       
\usepackage{amsfonts}       
\usepackage{nicefrac}       
\usepackage{microtype}      
\usepackage{xcolor}         
\usepackage{graphicx,psfrag,amsmath,amsfonts,verbatim}
\usepackage{hyperref}       
\usepackage{url}            
\usepackage{booktabs}       
\usepackage{amsfonts}       
\usepackage{nicefrac}       
\usepackage{microtype}      
\usepackage{xcolor}         
\usepackage{amsmath}
\usepackage{amsthm}
\usepackage{amsmath,amsfonts,bm}
\usepackage{siunitx}
\usepackage{thmtools}
\usepackage{diagbox}
\usepackage{thm-restate}

\usepackage[capitalize,noabbrev]{cleveref}
\newcommand{\E}{\mathbb{E}}
\theoremstyle{plain}
\newtheorem{theorem}{Theorem}[section]

\theoremstyle{definition}
\newtheorem{definition}[theorem]{Definition}
\newtheorem{assumption}[theorem]{Assumption}
\theoremstyle{remark}

\theoremstyle{example}
\newtheorem{example}{Example}

\usepackage[textsize=tiny]{todonotes}

\usepackage{hyperref}
\definecolor{cvprblue}{rgb}{0.21,0.49,0.74}
\definecolor{citecolor}{HTML}{2980b9}
\definecolor{linkcolor}{HTML}{c0392b}
\hypersetup{
    colorlinks=true,
    linkcolor=linkcolor,
    citecolor=citecolor,     
    urlcolor=black,
}

\begin{document}

\twocolumn[
\icmltitle{Understanding Nonlinear Implicit Bias via Region Counts in Input Space}



\icmlsetsymbol{equal}{*}

\begin{icmlauthorlist}
\icmlauthor{Jingwei Li}{equal,tsinghua,sqz}
\icmlauthor{Jing Xu}{equal,tsinghua}
\icmlauthor{Zifan Wang}{tsinghua,sqz}
\icmlauthor{Huishuai Zhang}{pku}
\icmlauthor{Jingzhao Zhang}{tsinghua,sqz}

\end{icmlauthorlist}

\icmlaffiliation{tsinghua}{Institute for Interdisciplinary Information Sciences,
Tsinghua University}
\icmlaffiliation{pku}{Wangxuan Institute of Computer Technology, Peking University}
\icmlaffiliation{sqz}{Shanghai Qizhi Institute}

\icmlcorrespondingauthor{Jingwei Li}{ljw22@mails.tsinghua.edu.cn}
\icmlcorrespondingauthor{Jingzhao Zhang}{jingzhaoz@mail.tsinghua.edu.cn}

\icmlkeywords{Machine Learning, ICML}

\vskip 0.3in
]



\printAffiliationsAndNotice{\icmlEqualContribution}  

\begin{abstract}
One explanation for the strong generalization ability of neural networks is implicit bias. Yet, the definition and mechanism of implicit bias in non-linear contexts remains little understood. In this work, we propose to characterize implicit bias by the count of connected regions in the input space with the same predicted label. Compared with parameter-dependent metrics (e.g., norm or normalized margin), region count can be better adapted to nonlinear, overparameterized models, because it is determined by the function mapping and is invariant to reparametrization. Empirically, we found that small region counts align with geometrically simple decision boundaries and correlate well with good generalization performance. We also observe that good hyper-parameter choices such as larger learning rates and smaller batch sizes can induce small region counts. We further establish the theoretical connections and explain how larger learning rate can induce small region counts in neural networks. 
\end{abstract}

\section{Introduction}
One mystery in deep neural networks lies in their ability to generalize, despite having significantly more learnable parameters than the number of training examples~\citep{zhang2017understanding}. The choice of network architectures, including factors such as nonlinearity, depth, and width, along with training procedures like initialization, optimization algorithms, and loss functions, can result in vastly diverse generalization performance~\citep{sutskever2013importance, smith2017don, wilson2017marginal, li2019towards}. The varied generalization abilities exhibited by neural networks are often explained by many researchers through the theory of \textit{implicit bias}~\citep{brutzkus2017sgd, soudry2018implicit}. Implicit bias refers to the inherent tendencies in how neural networks learn and generalize from the training data, even without explicit regularizations or constraints.

The implicit bias of linear neural networks has been extensively studied. One of the classical setting is linear classification with logistic loss. \citet{brutzkus2017sgd, soudry2018implicit, arora2019implicit} show that the parameter converges to the direction that maximizes the $L_2$ margin. For regression problems, it is proved that gradient descent or stochastic gradient descent converges to a parameter that is closest to the initialization in terms of $L_2$ norm~\citep{gunasekar2018characterizing}. The results of the linear regression model can be extended to deep linear neural networks by generalizing the definition of min-norm and max-margin solutions~\citep{ji2018gradient, vaskevicius2019implicit, woodworth2020kernel}.

Compared to linear models, defining implicit biases in non-linear networks poses significant challenges. One line of work studies homogeneous networks and demonstrates that gradient flow solutions converge to a KKT point of the max-margin problem~\citep{lyu2019gradient, ji2020directional, wang2021implicit,jacot2022feature}. Further research extends this analysis, showing that gradient flow converges to a max-margin solution under various norms~\citep{ongie2019function, chizat2020implicit}. Other studies focus on describing the implicit bias of neural networks using sharpness, such as~\citep{foret2020sharpness, montufar2022sharp, andriushchenko2023modern}.

\begin{figure*}[t]
  \centering
    \includegraphics[scale = 0.525]{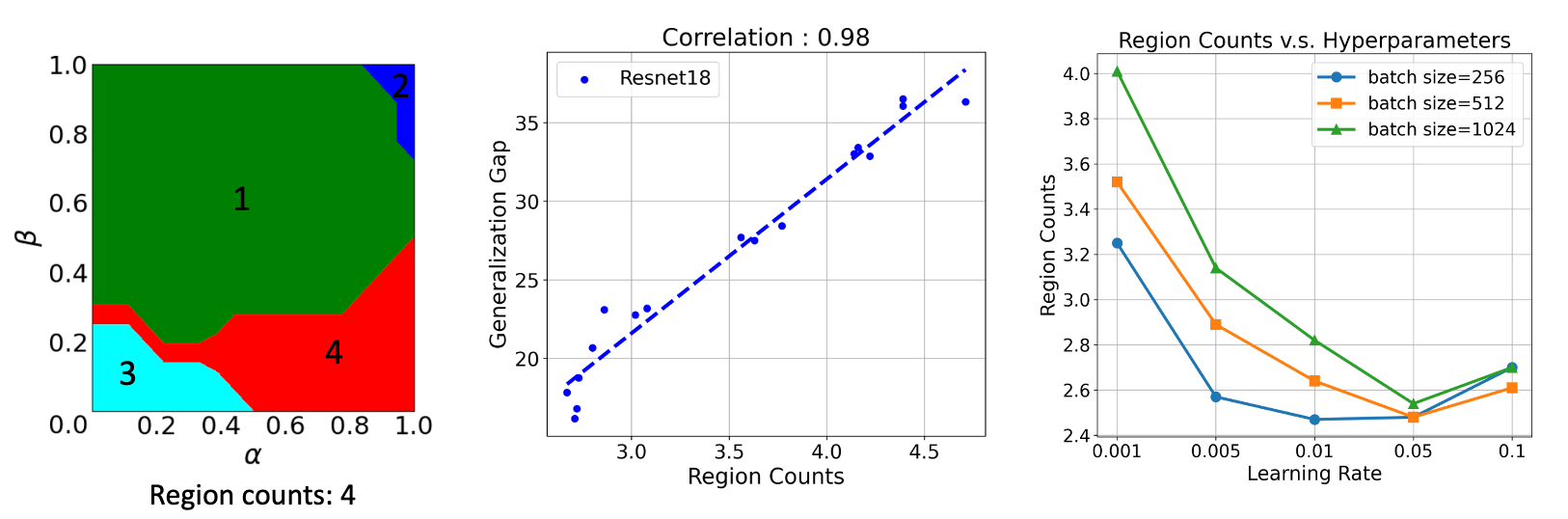}
        \caption{\textbf{A schematic illustration of main results in this paper.} Left: The region counts in 2-dimension input space. Each distinct region represents an area where the neural network makes the same prediction for all points within that region. Middle: A strong correlation between region counts and the generalization gap. Right: Larger learning rate or smaller batch size induces smaller region counts.
}\label{fig:teaser}\vspace{-0.4cm}
\end{figure*}

We note that previous definitions of implicit bias in neural networks mostly focus on certain metrics of network parameters. While these approaches enable explicit analyses of training trajectories, they encounter new challenges when applied to nonlinear networks: reparametrization of the network can preserve the function mapping but give completely different parameters, and consequently, different implicit biases.  We will discuss this point in detail in Section~\ref{sec:limit}.

Motivated by the above studies, we focus on leveraging decision boundaries in the input space to characterize implicit bias. We adopt a metric called region count, which measures the average number of connected components in the decision regions of a predictor (see Figure~\ref{fig:teaser}). This idea was initially introduced by \citet{somepalli2022can}, who used it under the name fragmentation score to study the double descent phenomenon by counting decision regions on two-dimensional planes defined by training examples. In contrast, we extend this idea beyond 2D subspaces and generalize the region count to arbitrary-dimensional input subspaces, which enables a more comprehensive geometric characterization.

To compute region count, we project the input space onto low-dimensional subspaces, use the trained network to predict the labels of points within these subspaces, and count the number of connected regions that share the same predicted label. Unlike previous studies on linear region count~\citep{hanin2019complexity, hanin2019deep, safran2022effective}, which focus on activation patterns and representational capacity, the metric of region count is label-dependent and reflects the functional behavior of the classifier. We find that region count correlates strongly with the generalization gap—defined as the difference between training and test errors. As illustrated in Figure~\ref{fig:teaser}, models with fewer decision regions tend to generalize better.

Furthermore, we show that neural networks trained with large learning rate or small batch size, which are typically deemed beneficial for generalization, are biased towards solutions that have small region counts. Therefore, region count empirically serves as an accurate generalization metric as well as an implicit bias indicator. We also provide theoretical analyses to explain this phenomenon. We prove that for two-layer ReLU neural networks, gradient descent with large learning rate induces a small region count, which accords well with our empirical findings. 

The main contributions of this paper are listed as follows: 
\begin{enumerate}
    \item We use the region count to to systematically characterize implicit bias. Through extensive experiments, we verify a strong correlation between region count and the generalization gap. This correlation remains robust across different learning methods, datasets, training parameters, and counting methods.
    \item We assess the factors that induce small region count, discovering that training with larger learning rates and smaller batch sizes typically results in fewer regions. This provides a possible cause for the implicit bias in neural networks.
    \item We conduct theoretical analyses on region counts, and show that for two-layer ReLU neural networks, gradient descent with large learning rate induces a small region count.
\end{enumerate}

\section{Related Works}
\paragraph{Implicit Bias of Linear Neural Network}
The implicit bias in linear neural networks are thoroughly investigated in recent works. For linear logistic regression on linearly separable data, full-batch gradient descent converges in the direction of the maximum margin solution~\citep{soudry2018implicit}. This foundational work has various follow-ups, including extensions to  non-linearly-separable data~\citep{ji2018risk, ji2019implicit}, stochastic gradient descent~\citep{nacson2019stochastic}, and other loss functions and optimizers~\citep{gunasekar2018characterizing}.

These findings in linear logistic regression are generalized to deep linear networks. For fully-connected neural networks with linear separable data, \citet{ji2018gradient} shows that the direction of weight also converges to $L_2$ max-margin solution. For linear diagonal networks, the gradient flow maximizes the margin with respect to a specific quasi-norm that is related to the depth of network~\citep{gunasekar2018implicit, woodworth2020kernel, pesme2021implicit}, leading to a bias towards sparse linear predictors as the depth goes to infinity. This sparsity bias also exists in linear convolutional networks~\citep{gunasekar2018implicit, yun2020unifying}.

\paragraph{Implicit Bias of Non-linear Neural Network}
The non-linearity of modern non-linear neural networks pose challenges to studying its implicit bias. Initial works in this area~\citep{lyu2019gradient, ji2020directional} focus on homogeneous networks. These studies show that with exponentially-tailed classification losses, both gradient flow and gradient descent converge directionally to a KKT point in a maximum-margin problem. \citet{wang2021implicit} consider a more general setup that includes different optimizers and prove that both Adam and RMSProp are capable of maximizing the margin in neural networks while Adagrad is not. \citet{ongie2019function,chizat2020implicit} showcased a bias towards maximizing the margin in a variation norm for infinite-width two-layer homogeneous networks. \citet{lyu2021gradient, jacot2022feature} identified margin maximization in two-layer Leaky-ReLU networks trained with linearly separable and symmetric data. More recent investigations into non-linear neural networks, such as~\citep{jacot2022implicit}, focus on the homogeneity of the non-linear layer, demonstrating an implicit bias characterized by a novel non-linear rank.

\paragraph{Region Counts of Neural Network}
Prior work has extensively studied the number of linear regions in ReLU networks~\citep{hanin2019complexity, hanin2019deep}, where a linear region refers to a set of inputs sharing the same activation pattern. For example, \citet{safran2022effective} showed that for a two-layer ReLU network with width $r$, gradient flow converges directionally to a network with at most $O(r)$ linear regions. Other studies~\citep{serra2018bounding, cai2023getting} demonstrated that increasing the number of linear regions can enhance representational capacity and potentially improve performance. However, linear regions are defined independently of output labels and thus primarily reflect the expressiveness of the network rather than its generalization behavior.

In contrast, decision regions are defined as connected subsets of the input space that correspond to the same predicted label. \citet{nguyen2018neural} investigated decision regions over the full input space and showed that, under certain conditions, each class tends to form a single connected region. However, connectivity in the full space does not necessarily imply connectivity in a subspace—for instance, two points may be connected in 3D but disconnected in a 2D slice. Subspace analysis thus reveals a richer and more nuanced structure of decision boundaries.

The paper~\citep{somepalli2022can} is the most relevant to our work, which is the first to systematically analyze the number of decision regions in input subspaces. They introduce the fragmentation score, defined as the average number of decision regions over 2D planes determined by triplets of training points. Their study focuses on illustrating how fragmentation varies with network width and its connection to the double descent phenomenon. While their approach provides important empirical observations, the analysis remains largely qualitative and limited to specific architectural variations.

Our work extends this line of research in several key directions. First, we adopt a more general definition of region count over arbitrary-dimensional subspaces, rather than restricting to two-dimensional planes. This allows us to capture more comprehensive geometric information and supports quantitative analysis. Second, we systematically investigate the correlation between region count and the generalization gap across a variety of architectures, training hyperparameters, and datasets. Third, we further offer a theoretical analysis of region count and study its behavior under distribution shifts. To the best of our knowledge, we are the first to provide a comprehensive empirical and theoretical analysis that links region count to generalization ability.

\section{Motivation}
\begin{figure*}[t]\centering
\begin{tabular}{ccc}
\includegraphics[scale=0.48]{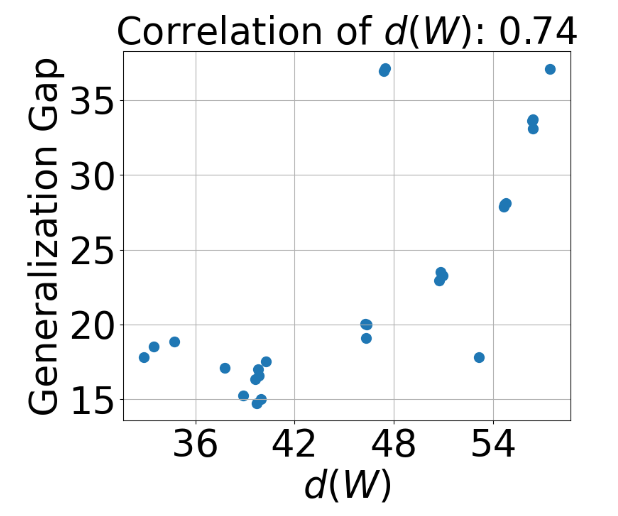} &
\includegraphics[scale=0.479]{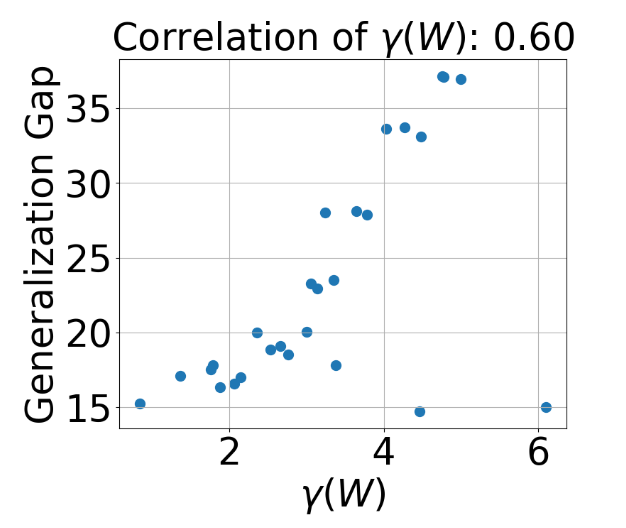} &
\includegraphics[scale=0.485]{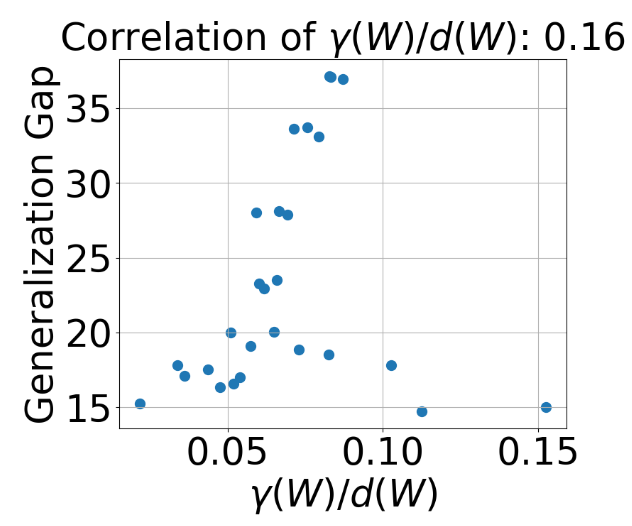} \\
\end{tabular}
\caption{\textbf{Norm-based and margin-based measures may not be predictive of generalization gaps.} We train ResNet18 on the CIFAR-10 dataset using various hyperparameters. These implicit bias measures can be ineffective for general non-linear neural networks.} \label{fig:limit} \vspace{-0.3cm}
\end{figure*}
\label{sec:limit}
Norm-based and margin-based characterizations belong to the most popular measures of implicit bias. Various definitions for norm and margin exist. For simplicity, we consider the following two definitions.
\begin{example}[Norm and Margin]
    Let $W=\{W_1,\cdots,W_l\}$ denote the post-training weight parameters of an $l$-layer neural network $f_W(x)=W_l\sigma(W_{l-1}\cdots W_2\sigma(W_1 x))$, with $\sigma(\cdot)$ as the ReLU activation function. Denote the weight initialization as $W^0=\{W_1^0,\cdots,W_l^0\}$. Consider the Frobenious norm between network weights and initialization: $$d(W) = \sqrt{\sum_{i=1}^l \|W_i-W_i^0\|_F^2},$$ and the output-space margin $$\gamma(W)=\mathbb{E}_{(x,y) \in D_{train}}\left[f(x)_y - \max_{i\neq y} f(x)_i\right].$$
\end{example}
$d(\cdot)$ and $\gamma(\cdot)$ are commonly used indicator for implicit bias of linear models~\citep{soudry2018implicit,ji2018risk}. 
However, both of them are not invariant to network reparameterization. We can construct a different set of network weight parameters by scaling the parameters as $\hat{W}=\{2W_1,\frac{1}{2}W_2, \cdots,W_l\}$, such that $f_W=f_{\hat{W}}$ but $d(W)\neq d(\hat{W})$ in general. 
Similarly, we can scale the last layer weights and get $\Tilde{W}=\{W_1,\cdots,2W_l\}$, such that $\gamma(\Tilde{W})\neq\gamma(W)$, but $\text{argmax}_i f_W(x)=\text{argmax}_i f_{\Tilde{W}}(x)$. 
This reparameterization trick also works for more complicated norm-based and margin-based generalization metric in~\citep{jiang2019fantastic}, or the sharpness-based metrics~\citep{andriushchenko2023modern}.

We numerically investigate whether they are effective measures, by training a ResNet18 on CIFAR-10 dataset, using different hyperparameters as in Table~\ref{tab:hyperparameter}. The results are presented in Figure~\ref{fig:limit}, which indicates that these measures have a low correlation with the generalization gap in the deep learning regime. One could choose other definitions of norms to achieve stronger correlations, but such choices are often problem-specific and require domain expertise, as discussed in~\citep{jiang2019fantastic}.

The definition of margin may also be improved to the input-space margin, \emph{i.e.}, the $\ell_2$ distance of input data $x$ to decision boundary defined by the classifier, which is able to characterize the quality and robustness of the classifier. 
This metric is invariant to reparameterization and therefore more intrinsic to the underlying classifier. 
However, due to the highly nonconvex loss landscape, the input-space margin is NP-complete to compute and even hard to approximate~\citep{katz2017reluplex, weng2018towards}. Therefore, quantitatively analyzing the decision boundary of a neural network and characterizing its implicit bias remains a challenge. 

Our motivation stems from a simple idea: although the margin in the input space is difficult to compute directly, we can instead quantify the number of regions partitioned by the decision boundary. This measure is invariant to model reparameterization and reflects the geometric complexity of the classifier, which are crucial ingredients for a robust and scalable indicator of implicit bias.

\section{Preliminary}

Although region count is a natural measure for the complexity of a predictor, and it depends only on the decision function rather than the model parameterization, its formal definition and computability remains unclear.   In this section, we first provide the definition and low-dimensional approximation of region counts. We then empirically verify that region counts correlate with generalization gap.

\subsection{Definition of Region Counts}

Let $d$ denote the training data dimension and $f: \mathbb{R}^d \rightarrow \{1,2,\ldots,N\}$ denote a neural network for a classification task with $N$ classes. For a subset $U \subset \mathbb{R}^d$, we can define the connectedness of its element as follows:
\begin{definition}[Connectedness]
 We say the data points $x_1, x_2 \in U$ are (path) connected with respect to a neural network $f$ if they satisfy:
    \begin{itemize}
        \item $f(x_1) = f(x_2) = c$,
        \item There exist a continuous mapping $\gamma: [0,1] \rightarrow U$, $\gamma(0) = x_1$, $\gamma(1) = x_2$, and for any $t \in [0,1]$, $f(\gamma(t)) = c$. 
    \end{itemize}
\end{definition}
Then we define the connected region in this subset:
\begin{definition}[Maximally Connected Region]
    We say $V\subset U$ is a maximally connected region in $U\in \mathbb{R}^d$ with respect to a neural network $f$ if it satisfies the following property:
    \begin{itemize}
        \item For any $x,y \in V$, they are connected.
        \item For any $x \in V$, $y \in U \setminus V$, they are not connected.
    \end{itemize}
\end{definition}

 Finally, we formally define the region count as follows:

\begin{definition}[Region Count]
For a subset $U\subseteq \mathbb{R}^d$, we define its region count $R_U$ as the number of maximally connected regions in $U$ with respect to a neural network $f$:
\begin{align*}
    R_U = \text{card}\{V \subset U | V \text{ is a maximally connected region}\} \,,
\end{align*}
where card is the cardinality of a set.
\end{definition}

\subsection{Estimating Region Counts}
Calculating the region count in the original high-dimensional input space can be computationally intractable. We follow the estimation method in~\citep{somepalli2022can}, use a computationally efficient surrogate by calculating the region counts on low dimensional subspace spanned by training data points. 

\begin{definition}[Region Count in $d$-Dimensional Subspace]
    We randomly sample $d+1$ datapoints in the training set $D_{train}$ to generate a convex region in $\mathbb{R}^d$ subspace.
    The $d$-dimensional region count $R_d$ is defined as the expectation of number of maximally connected regions:
    $$R_d = E_{x_1,x_2, \ldots, x_{d+1} \sim D_{train}}[R_{\text{Conv}\{x_1,x_2, \ldots, x_{d+1}\}}]\,,$$
    where $x_1,x_2, \ldots, x_{d+1}$ are sampled from the training dataset, and $\text{Conv}\{x_1,x_2, \ldots, x_{d+1}\}$ is the convex hull formed by these $d+1$ points. 
\end{definition}

 \begin{figure}[h]
  \centering
  \includegraphics[width=0.3\textwidth]{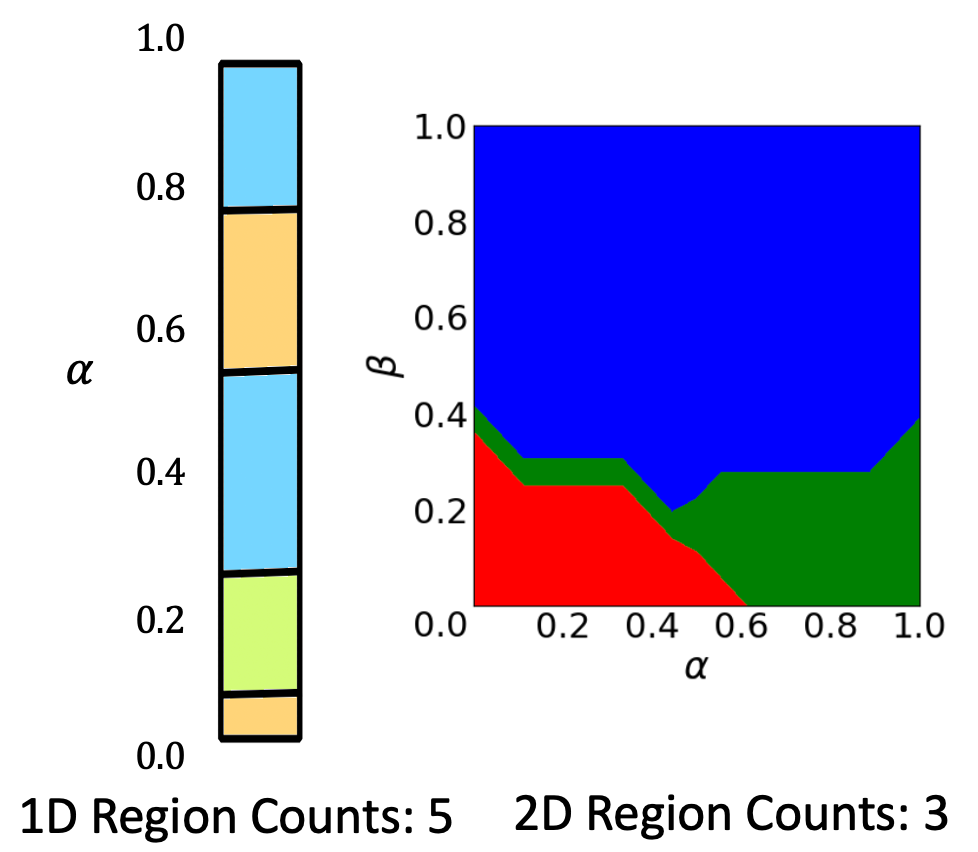}  
  \caption{\textbf{Illustrations of region counts in 1D and 2D subspace.} We use different colors to represent different outputs of the neural network. }
    \label{fig:dataset}
    \vspace{-0.1cm}
\end{figure}

This paper primarily focuses on low dimension spaces, which is illustrated as below.
In practice, we randomly sample training data points for multiple times and take the average region counts.
In Section~\ref{sec:ablation}, we show that the choice of subspace dimension $d$ does not significantly affect the results. The details on how to count the regions and generate the polytopes are provided in Appendix~\ref{app:count}. 

\begin{example}[Region counts in 1D and 2D subspace]
    \label{exp:1d}
    For region count in 1-dimensional subspace, we randomly sample two data points, denoted as $x_1$ and $x_2$, from the training set, and
    calculate the region count on the line segment connecting them:
    $$\{\alpha x_1 + (1-\alpha) x_2 \,, 0 \le \alpha \le 1\}.$$
    For the 2-dimensional case, we randomly sample three data points, $x_1$, $x_2$, and $x_3$, from the training set, and calculate the region count in the convex hull spanned by them:
    \begin{align*}
        \{ \alpha x_1 + \beta x_2 + (1-\alpha-\beta)x_3\,, \alpha \ge 0, \beta \ge 0, \alpha+\beta \le 1 \}\,.
    \end{align*}
We provide an illustration in Figure~\ref{fig:dataset}.
\end{example}

\section{Region Counts Correlate with Generalization Gaps}
\label{sec:def}
In this section, we present our major empirical findings, which reveal a strong correlation between region counts and the generalization error of neural networks\footnote{Code is available at
\href{https://github.com/lijingwei0502/implicit_bias}
{\nolinkurl{https://github.com/lijingwei0502/implicit_bias}}.}. 

We conduct image classification experiments on the CIFAR-10 dataset, using different architectures, including ResNet18~\citep{he2016deep}, EfficientNetB0~\citep{tan2019efficientnet}, and SeNet18~\citep{hu2018squeeze}. 
Results on other architectures are deferred to ablation studies. 
We vary the hyperparameters for training, such as learning rate, batch size and weight decay coefficient, whose numbers are reported in  Table~\ref{tab:hyperparameter}. The region count is calculated using randomly generated 1D hyperplanes, as described in Example~\ref{exp:1d}. We run each experiment 100 times and report the average number.

We plot the region count and generalization gap of different setups in Figure~\ref{fig:start}, and calculate the correlation between them. For each network architecture, we observe a strong correlation as high as 0.98. The overall correlation for all the three networks still reaches 0.93. This reveals a remarkably high correlation between region counts and generalization gap.


\begin{table}[h]
\small
\centering
\captionsetup{skip=5pt}
\caption{\textbf{The hyperparameters for experiments.} We vary the learning rate, batch size, and weight decay for training a neural network, to modulate the model's generalization ability.}
\label{tab:hyperparameter}

\begin{tabular}{cc}
\toprule  
 Hyperparameters & Value \\
\midrule
 Learning rate  & $0.1$, $0.01$, $0.001$ \\ 
  Batch size & $256$, $512$, $1024$ \\ 
 Weight decay & \num{e-5}, \num{e-6}, \num{e-7} \\
\bottomrule
\end{tabular}

\end{table}

\begin{figure}[ht]
  \centering
    \includegraphics[scale = 0.37]{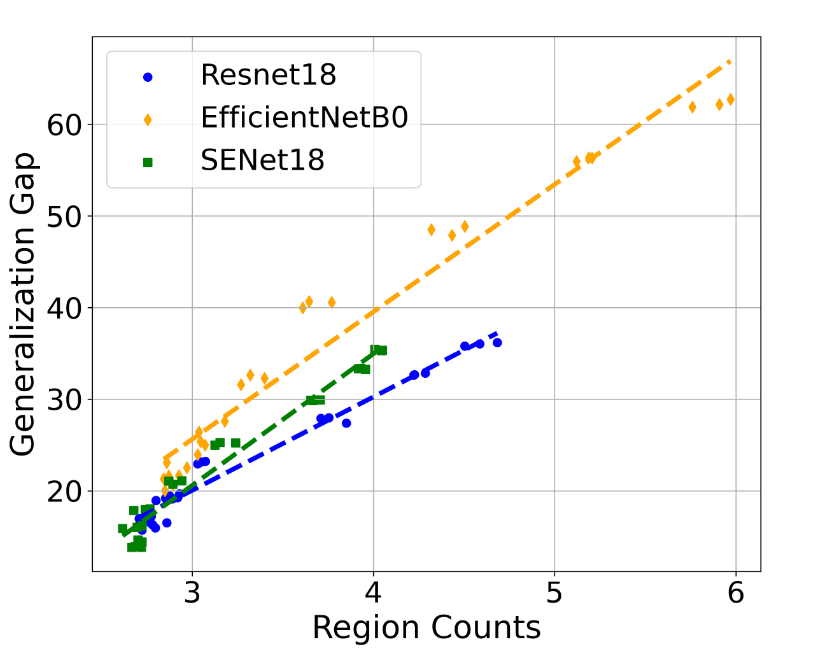}
    \caption{\textbf{Strong correlation between region counts and generalization gap.} We conduct experiments using three neural networks on the CIFAR-10 dataset, with various hyperparameters. There is a strong correlation between region counts and the generalization gap, with a correlation coefficient of 0.98 for each network and 0.93 across all networks. }
    \label{fig:start}
\end{figure}





\section{Region Counts Quantify Implicit Bias}
In this section, we further investigate the implicit bias of neural networks via region counts. We show both empirically and theoretically that neural networks trained with appropriate hyperparameters tend to have smaller region counts, thus achieving better generalization performance. 

\subsection{The Bias from Training Hyperparameters}
Training neural networks requires careful selection of many hyperparameters, such as learning rates, batch sizes, optimizers, epochs and so on. Here, we primarily focus on learning rate and batch size, and study their impact on the region count. 

\paragraph{Learning Rates.} 
We provide the relationship between the learning rate and the region count in Figure \ref{fig:lr}. Our findings indicate that a larger learning rate tends to simplify the decision boundary and results in a smaller region count in the hyperplane. This accords well with real practices, where large learning rates of $0.1$ or $0.01$ are often favored for better generalization.

\paragraph{Batch Sizes.} Similarly, the training batch size can impact the number of regions. As shown in Figure \ref{fig:lr}, smaller batch sizes lead to a model with fewer regions. This result reveals the advantage of small-batch training, which leads to better generalization accuracy. 
\begin{figure*}[t]
  \centering
    \includegraphics[scale = 0.3]{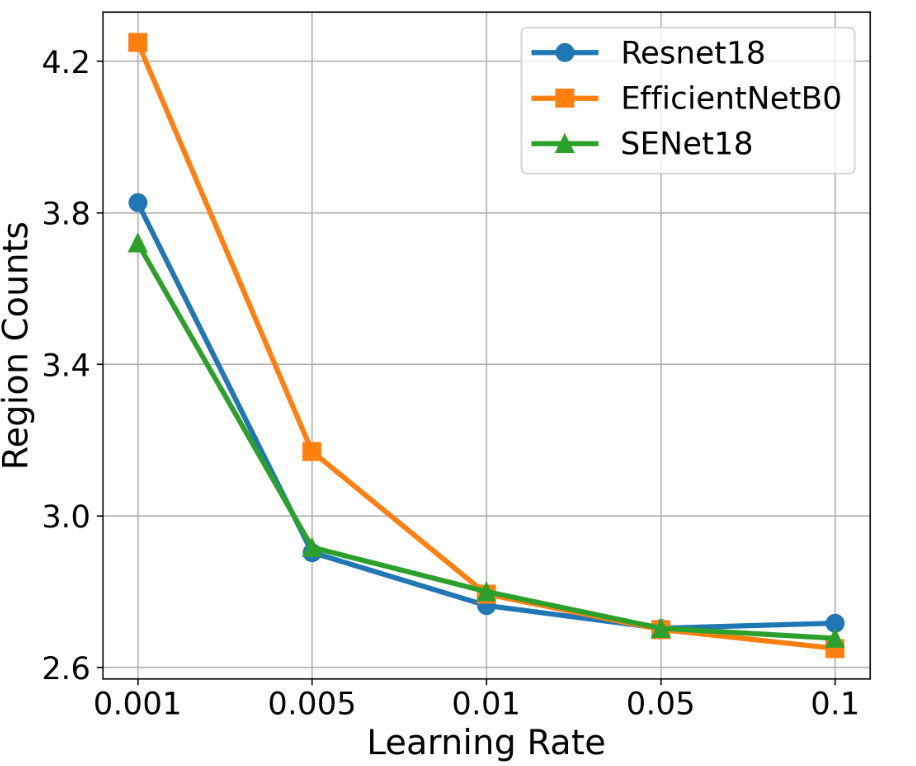}\hspace{1.8cm}
    \includegraphics[scale = 0.3]{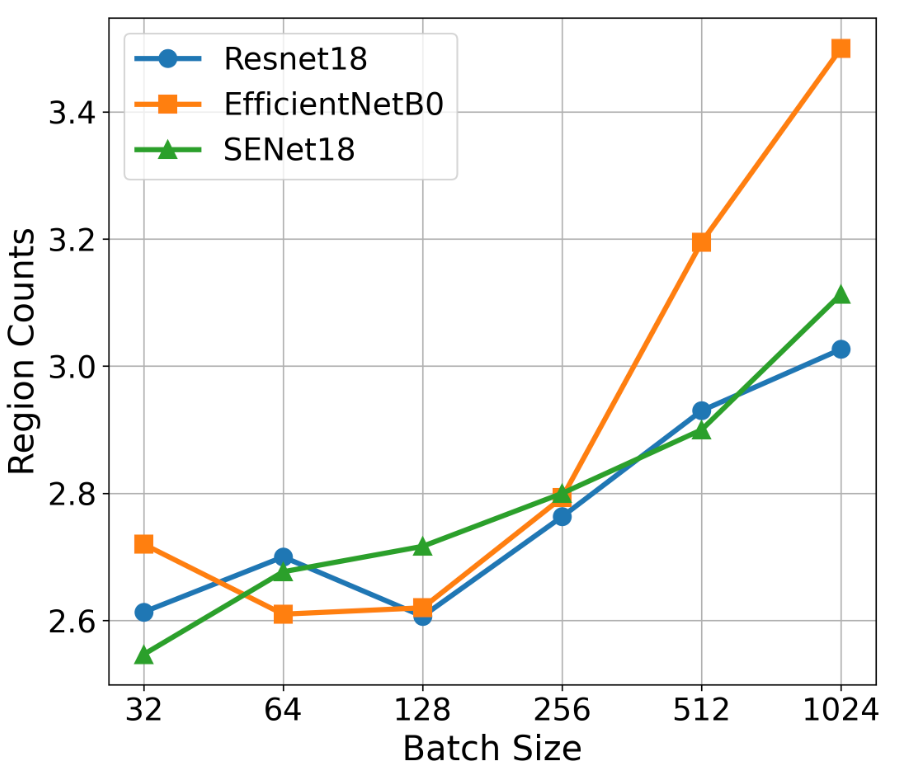}
    \caption{\textbf{Large learning rate and small batch size reduce region counts.} We train three networks on the CIFAR-10 dataset, varying the batch sizes and learning rates. Our findings reveal that a smaller batch size or a higher learning rate results in smaller region counts, allowing the network to learn a simpler decision boundary and generalize better.}
    \label{fig:lr}
\end{figure*}

Previous studies~\citep{keskar2016large, jastrzkebski2017three, hoffer2017train, novak2018sensitivity, lewkowycz2020large} find that certain hyperparameters, such as a large learning rate and a small batch size, can improve the generalization of the neural network. Our observations provide a possible explanation: these good hyperparameter choices lead to a reduced region count. Such simplicity bias can decrease the generalization gap of neural networks.

\subsection{Theoretical Explanations}
Next, we present a theoretical analysis to explain why some hyperparameter choices, such as large learning rate, can lead to small region counts. 

Consider a two layer ReLU neural network $f_W(x)=\sum_{i=1}^p a_i \sigma(w_i^\top x)$. The second layer weights $a_i$ are initialized uniformly from $\{1, -1\}$ and fixed throughout training. Let $\mathcal{D}=\{(x_i, y_i)\}_{1\le i\le N}$ denote the training set. 
Consider training $f_W$ on $\mathcal{D}$ using gradient descent~(GD) with learning rate $\eta$. We choose the quadratic loss $l(W,x,y)=\frac{1}{2}(y-f_W(x))^2$ and denote $L(W)=\frac{1}{N}\sum_{i=1}^N l(W,x_i,y_i)$. 
Denote the GD trajectory as $\{W_t\}_{t\ge 0}$.
For two input data $x_a, x_b$, Let $R(x_a, x_b, W)$ denote the region count on the line segment connecting them, and $N(x_a, W)$ denote the number of activated neurons with input $x_a$, i.e., the number of $i$ such that $w_i^\top x_a>0$.  

We make the following assumption on the data distribution. 
\begin{assumption}\label{asm:data}
    The training dataset $\mathcal{D}=\{(x_i, y_i)\}_{1\le i\le N}$ satisfies the following two properties:
    \begin{enumerate}
        \item $\|x_i\|\ge r$ for all $1\le i\le N$,
        \item With probability one, any $W\in \{W_t\}_{t\ge 0}$ satisfies $w_i^\top x_j\neq 0$ for all $1\le i \le q, 1\le j\le N$, where the randomness comes from weight initialization.
    \end{enumerate}
\end{assumption}

The validity of Assumption~\ref{asm:data} comes from the fact that the bifurcation zone~\citep{bertoin2021numerical} of ReLU neural networks, which contains its non-differentiable points, has Lebesgue measure zero~\citep{bolte2020mathematical, bolte2021conservative, bianchi2022convergence}. Therefore, if the distribution of weights are absolutely continuous with respect to the Lebesgue measure, the bifurcation zone can be avoided with probability one. We conjecture that it can be proved rigorously, but leave it as an assumption since the proof diverges from the main content in this paper. 

The next assumption characterizes the sharpness along the training trajectory. This is actually from  the well-known edge of stability phenomenon~\citep{cohen2020gradient, damian2022self, arora2022understanding, ahn2024learning}, which  states that the sharpness of neural networks, characterized by the $\ell_2$ norm of the Hessian matrix, hovers around $\frac{2}{\eta}$.

\begin{assumption}[Edge of Stability]\label{asm:eos}
    There exist a $T\in\mathbb{N}$, such that for $t\ge T$, with we have
    \begin{align*}
        \lambda_{\max}(\nabla_W^2 L(W_t))=\Theta\left(\frac{1}{\eta}\right),
    \end{align*}
    where $\lambda_{\max}$ denotes the maximum eigenvalue of a matrix. 
\end{assumption}

We are now ready to present the main theorem, which establishes a relationship between region count and learning rate. 
\begin{restatable}{theorem}{two}
Under Assumption~\ref{asm:data} and~\ref{asm:eos}, we have that for neural net weights $W_t$ at training step $t\ge T$, with probability one, the average region count $R(X, X', W_t)$ for random training data point $X, X'$ can be bounded as:
\begin{align*}
    &\E_{X, X'}[R(X, X', W_t)] \\
    =& \frac{1}{N^2}\sum_{i=1}^N \sum_{j=1}^N R(x_i, x_j, W_t)\le O\left(\frac{N}{r^2\eta}\right). 
\end{align*}
\label{thm:2}
\end{restatable}
The theorem demonstrates that with a larger learning rate, gradient descent has the implicit bias to yield solutions with smaller region counts. This aligns well with the previous observations.
\begin{table*}[ht]
\small
\caption{\textbf{Experimental consistency across networks, datasets, and counting methods.} We conduct experiments on various types of networks across multiple datasets. We also alter the method of calculating the region counts. The results of the correlation indicate that our findings are consistent across different setups.}
\centering
\renewcommand{\arraystretch}{1.35}
\begin{tabular}{ccccccccc}
\hline
\multirow{2}{*}{Network} & \multicolumn{3}{c}{Dataset} & & \multicolumn{4}{c}{Counting Dimension} \\ \cline{2-4} \cline{6-9}
   & \emph{CIFAR-10} & \emph{CIFAR-100} & \emph{ImageNet} & & 2 & 3 & 4 & 5 \\ \hline
ResNet18 & 0.98 & 0.96 & 0.91 & &  0.96 & 0.97 & 0.97 & 0.96 \\
ResNet34 & 0.98 & 0.98 & 0.82 & &  0.98 & 0.98 & 0.98 & 0.99 \\
VGG19 & 0.94 & 0.85 & 0.78 &&0.88 & 0.86& 0.84& 0.86 \\
MobileNet & 0.95 & 0.95 & 0.92 & & 0.99& 0.99 &0.99 & 0.99\\
SENet18 & 0.98 & 0.85 & 0.80 & &0.97 &0.97 & 0.97& 0.93 \\
ShuffleNetV2 & 0.95 & 0.92 & 0.92 && 0.94& 0.95 & 0.95&0.93 \\
EfficientNetB0 & 0.98 & 0.84 & 0.93 && 0.99 & 0.99& 0.99& 0.98 \\
RegNetX\_200MF & 0.98 & 0.87 & 0.97 &  &0.98 & 0.99&0.99 &0.98 \\
SimpleDLA & 0.98 & 0.94 & 0.84 && 0.99 & 0.99& 0.98& 0.99 \\
\hline
\end{tabular}
\label{tab:ablation}
\vspace{-0.2cm}
\end{table*}
We defer the proof of this theorem to Appendix~\ref{app:proof}, and sketch the proof as follows. 
The proof begins with bounding the region count using the activation pattern of ReLU neurons, as stated in the following lemma. 
\begin{restatable}{lemma}{lemregion}
The region counts between a pair of data points is upper-bounded by the number of active neurons.
For two inputs $x_a, x_b$, we have
    $$R(x_a, x_b, W)\le N(x_a, W)+N(x_b, W)+2.$$
\end{restatable}
Then we prove that the activation pattern gives a bound on the smoothness of the training loss. 

\begin{restatable}{lemma}{lemsmooth}
    The sharpness of a neural network is lower-bounded by the number of active neurons: $$\lambda_{\max}\left(\nabla^2_W L(W)\right) \ge \frac{r^2}{N^2} \sum_{i=1}^N N(x_i, W).$$
\end{restatable}
Note that this lemma brings in an additional $N$ in the denominator, which leads to a $N$-dependent bound in Theorem~\ref{thm:2}. The $N$
 dependency is actually tight in worst case analysis, by considering $N$ points on a line with alternating labels. We conjecture that the $N$-dependency can be optimized under further structural assumptions on the data distribution, and leave it for further investigations. Equipped with these two lemmas, Theorem~\ref{thm:2} is a consequence of the sharpness condition in Assumption~\ref{asm:eos}.

\section{Ablation Studies}
\label{sec:ablation}

This section presents an ablation study to validate the robustness and consistency of our findings. We systematically vary key aspects of our experimental setup, including the network architecture, dataset, optimizer, and the method of computing the plane, and  evaluate their impact on our main results of
the correlation between region count and the generalization gap.

\paragraph{More Architectures, Datasets and Hyperplane Dimensions.}
We first examine the influence of neural network architectures and datasets on our results. We provide additional results on various neural network architectures such as ResNet34~\citep{he2016deep}, VGG19~\citep{simonyan2014very}, MobileNetV2~\citep{sandler2018mobilenetv2}, ShuffleNetV2~\citep{ma2018shufflenet}, RegNet200MF~\citep{radosavovic2020designing}, and SimpleDLA~\citep{yu2018deep}. We also use various datasets such as CIFAR-100~\citep{krizhevsky2009learning} and ImageNet~\citep{deng2009imagenet}. Region counts and generalization gaps are evaluated across various learning rates, batch sizes, and weight decay parameters as listed in Table~\ref{tab:hyperparameter}. 

We also explore the effects of different methods for generating the hyperplane in the input space. In our previous experiments, we generate the 1-dimensional plane using random pairs of samples from the training set and calculate the region count on them. Here we explore region counts in higher dimensional planes, that are spanned by 2 to 5 data randomly-selected points from the training set, using the CIFAR-10 dataset.

The experiment results of the correlation are presented in Table~\ref{tab:ablation}.  We also provide correlation plots for each network in Appendix~\ref{app:addtional}. We observe that the strong correlation between region count and the generalization gap remains consistent in various setups.
The consistency indicates that our findings reveal a fundamental characteristic of non-linear neural networks.

We also provide evaluations by varying the optimizer and hyperplane generation algorithms. The results are deferred to Appendix~\ref{app:ablation}.

\paragraph{Data Augmentations.}

Mixup~\citep{zhang2017mixup} is a data augmentation technique that creates training samples by linearly interpolating pairs of input data and their corresponding labels. We train a ResNet-18 model using mixup, with other hyperparameters in Table~\ref{tab:hyperparameter}. The plot in Figure \ref{fig:mixup} illustrates that Mixup induces smoother decision boundaries with smaller region count and has a better generalization performance.   
Random crop and random horizontal flip is another way to enhance the diversity of the training dataset.
We apply random crop of size 32$\times$32 with padding 4 and random horizontal flip with a probability of 0.5 as data augmentations. 
As depicted in Figure \ref{fig:crop and flip}, we observe that compared with mixup, random crop and random flip result in a more evident vertical shift in the performance curve. 

Although mixup and random crop affect region count differently, the correlation between region count and generalization remains high after applying these data augmentation techniques. Detailed correlation results can be found in Appendix~\ref{app:ablation}. We will explore the underlying principles of this experimental phenomenon as a future direction.

\begin{figure}[h]
  \centering
    \includegraphics[scale = 0.34]{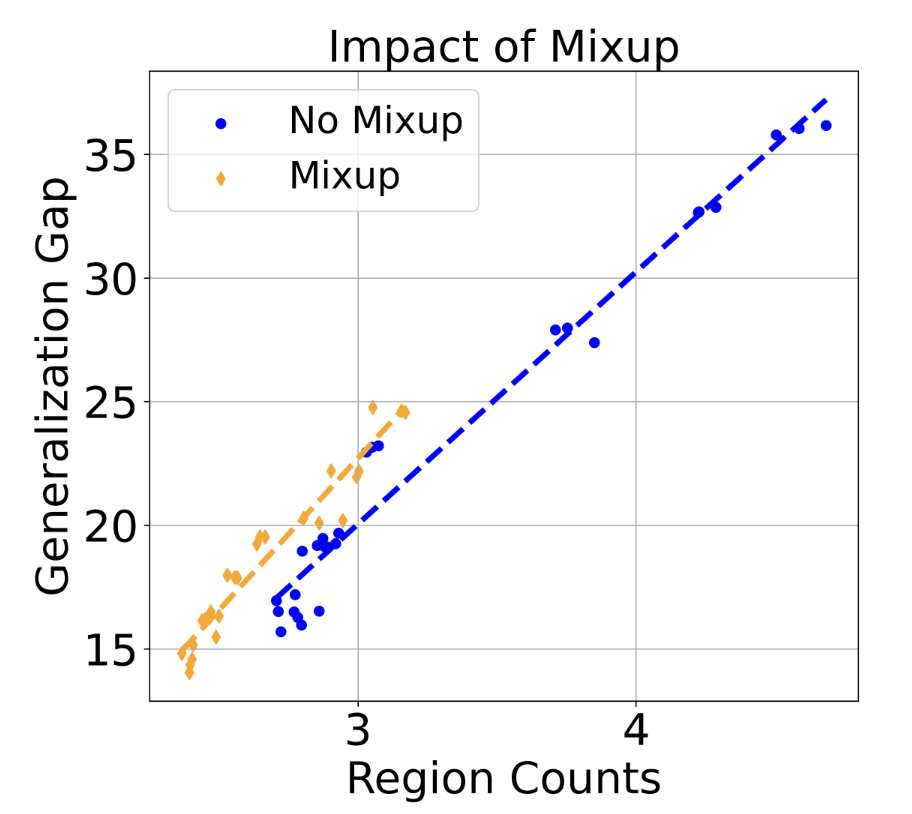}
    \caption{\textbf{The impact of mixup.} This figure shows that mixup improves the model's generalization ability and reduces the number of regions in the hyperplane.}
    \label{fig:mixup}
\end{figure}

\begin{figure}[h]
  \centering
    \includegraphics[scale = 0.399]{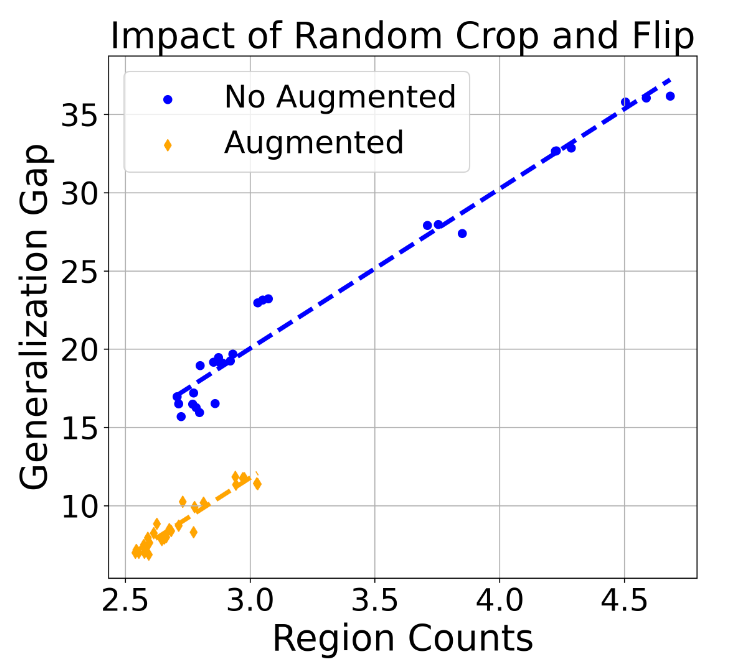}
    \caption{\textbf{The impact of random crop and random flip.} Unlike mixup, data augmentation results in a vertical shift in the performance curve, accompanied by a decrease in the number of regions and a more significant enhancement in test accuracy.}
    \label{fig:crop and flip}
\end{figure}

\section{Conclusions and Future Directions}
This paper introduces a novel approach to characterizing the implicit bias of neural networks. We study the region counts in the input space and identify its strong correlation with generalization gap in non-linear neural networks. These findings are consistent across various network architectures, datasets, optimizers. Our analysis offers a new perspective to quantify and understand the generalization property and implicit bias of neural networks. 

Our paper suggests several promising directions for future research. 
Firstly, our analyses of why large learning rate induces small region counts mainly focus on a simplified setup. The analyses for more general settings remain open.
Secondly, extending the definition of region count to non-classification tasks, such as natural language generation, would be a worthwhile direction.
Lastly, region count can be leveraged to design new architectures or regularization, that can potentially improve the generalization performance of neural networks.

\section*{Impact Statement}

This paper presents work whose goal is to advance the field of Machine Learning. There are many potential societal consequences of our work, none which we feel must be specifically highlighted here.


\bibliography{example_paper}
\bibliographystyle{icml2025}

\newpage
\appendix
\onecolumn
\section{Experiment details}
\label{app:A}
In this section, we provide the detailed experiment settings.

\subsection{Details on Architectures and Datasets}
We conduct experiments on different neural network architectures, including ResNet18 and ResNet34~\citep{he2016deep}, EfficientNetB0~\citep{tan2019efficientnet}, SENet18~\citep{hu2018squeeze}, VGG19~\citep{simonyan2014very}, MobileNetV2~\citep{sandler2018mobilenetv2}, ShuffleNetV2~\citep{ma2018shufflenet}, RegNet200MF~\citep{radosavovic2020designing}, SimpleDLA~\citep{yu2018deep}. We conduct all experiments using NVIDIA RTX 6000 graphics card.

We use CIFAR-10/100~\citep{krizhevsky2009learning} and Imagenet-1k~\citep{deng2009imagenet} as datasets.
For CIFAR-10 and CIFAR-100 dataset, each network was trained for $200$ epochs using the Stochastic Gradient Descent (SGD) algorithm with cosine learning rate schedule. We choose 27 combinations of hyperparameters in Table~\ref{tab:hyperparameter}, and for each hyperparameter we use 3 random seeds and report the average metrics. 
For the Imagenet-1k dataset, each network was trained for $50$ epochs with random data crop and random flip. We use the same optimizer and 27 combinations of hyperparameters as in CIFAR-10 and CIFAR-100 experiments. It is worth noting that we make minor adjustments on hyperparameters for certain networks to ensure stable training. For example, in the case of VGG19, the training is unable to converge when the learning rate is set to 0.1; therefore, we adjust it to 0.05.




\subsection{Correlation Plots}
\label{app:addtional}
We show the correlation plot of average regions and test accuracy in Figure~\ref{fig:all1}. The figure consists of results from training different networks on CIFAR-10 dataset with SGD for 200 epochs, using the hyperparameters specified in Table~\ref{tab:hyperparameter}. 
The results show that different networks have different number of regions, ranging from 2 to 20. However, the correlation of test accuracy and average number of regions are consistently high in all the networks. 
\begin{figure}[h]
    \centering
    \begin{minipage}[b]{0.312\linewidth}
        \includegraphics[width=\linewidth]{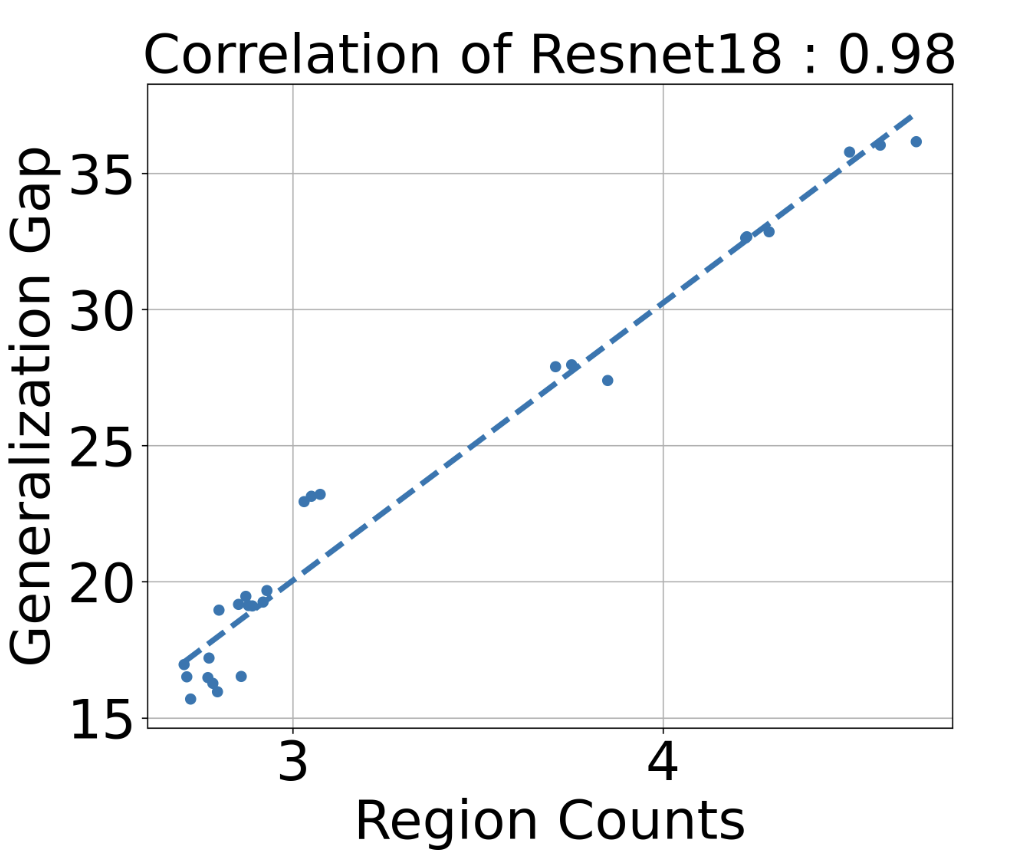}
        \centering
    \end{minipage}
    \hfill
    \begin{minipage}[b]{0.317\linewidth}
        \includegraphics[width=\linewidth]{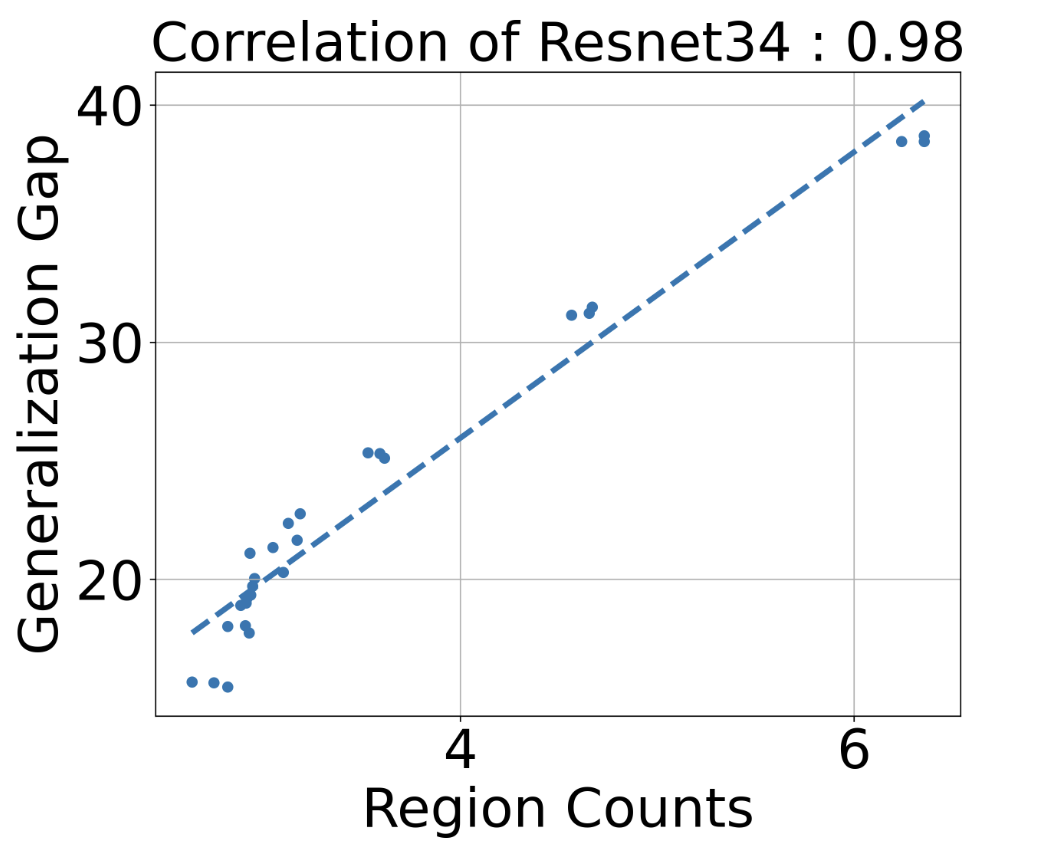}
        \centering
    \end{minipage}
    \hfill
    \begin{minipage}[b]{0.32\linewidth}
        \includegraphics[width=\linewidth]{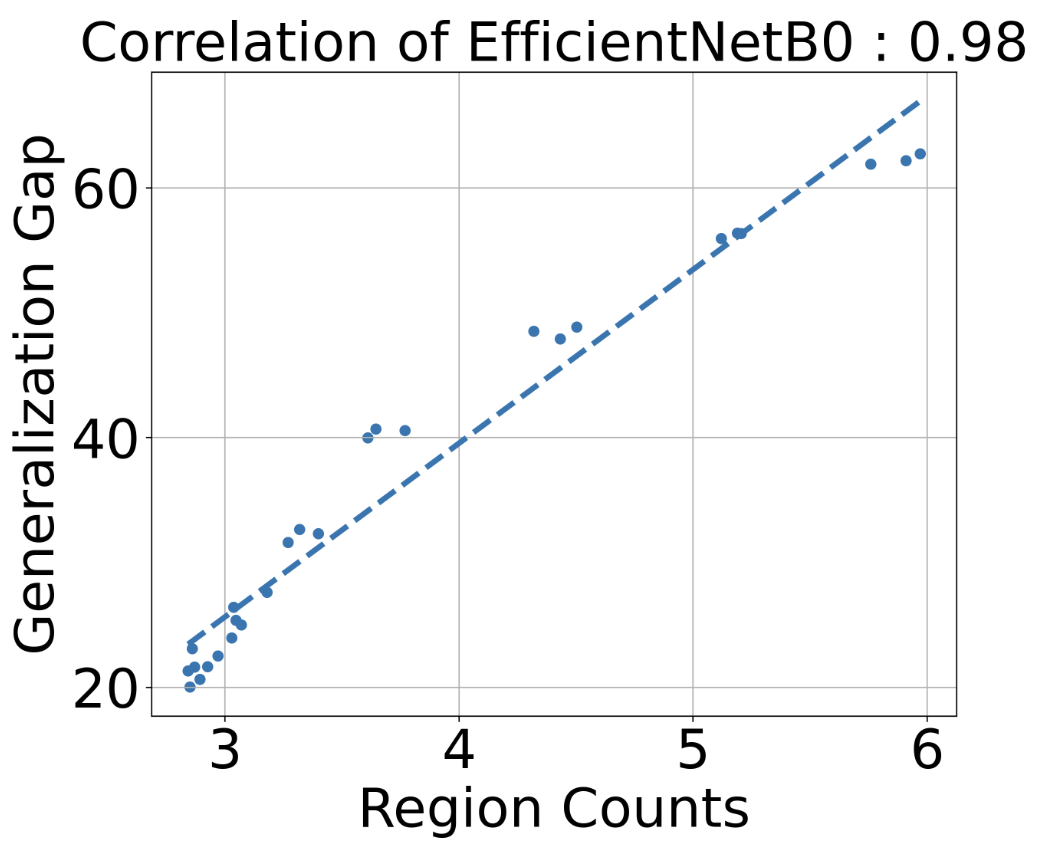}
        \centering
    \end{minipage}

    \begin{minipage}[b]{0.31\linewidth}
        \includegraphics[width=\linewidth]{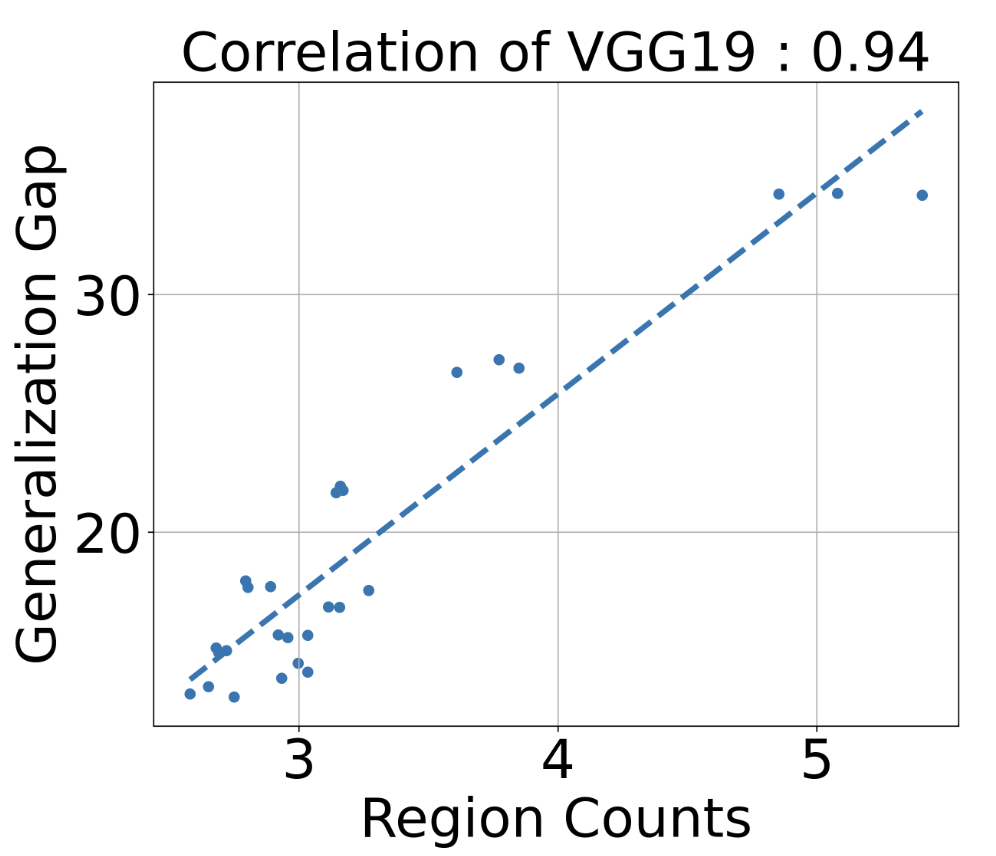}
        \centering
    \end{minipage}
    \hfill
    \begin{minipage}[b]{0.317\linewidth}
        \includegraphics[width=\linewidth]{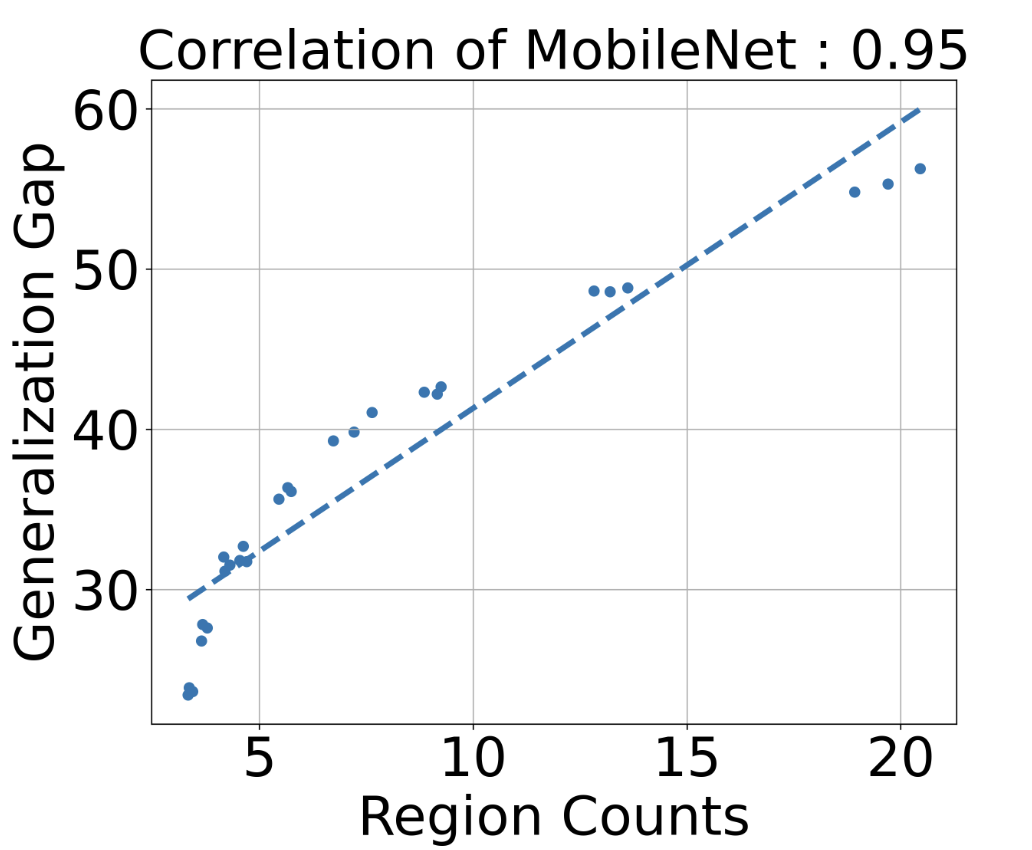}
        \centering
    \end{minipage}
    \hfill
    \begin{minipage}[b]{0.32\linewidth}
        \includegraphics[width=\linewidth]{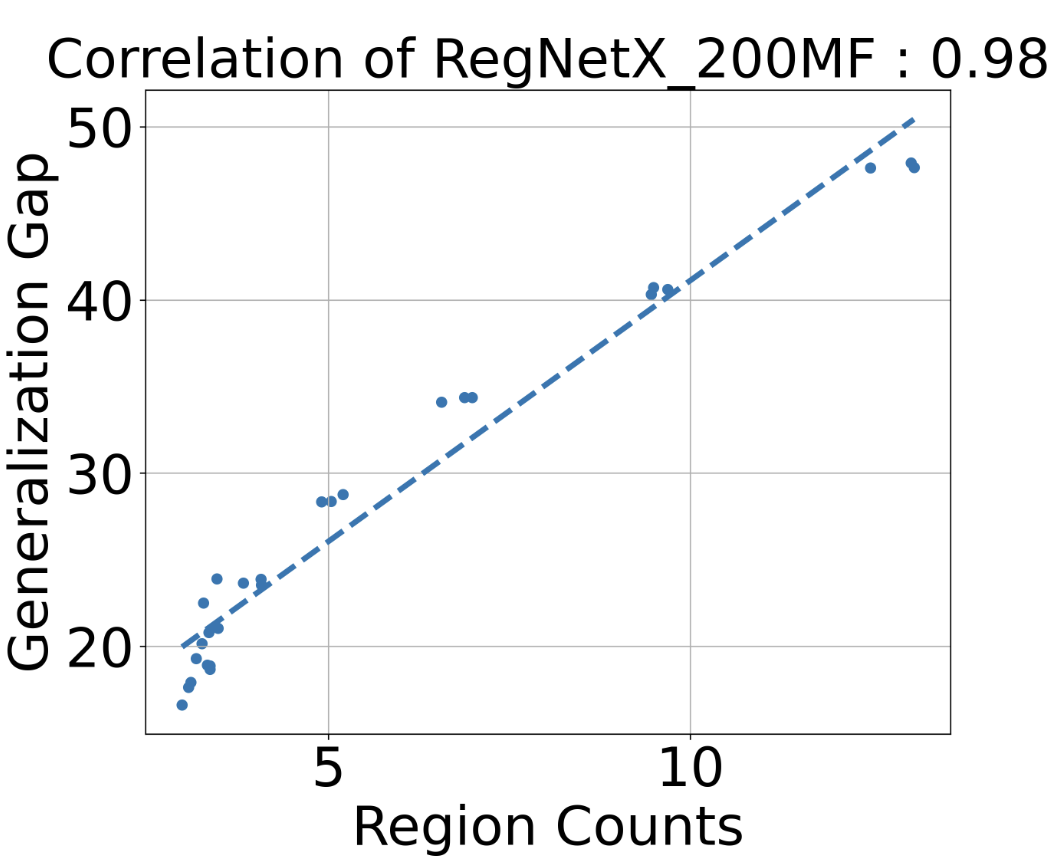}
        \centering
    \end{minipage}
    
    \begin{minipage}[b]{0.314\linewidth}
        \includegraphics[width=\linewidth]{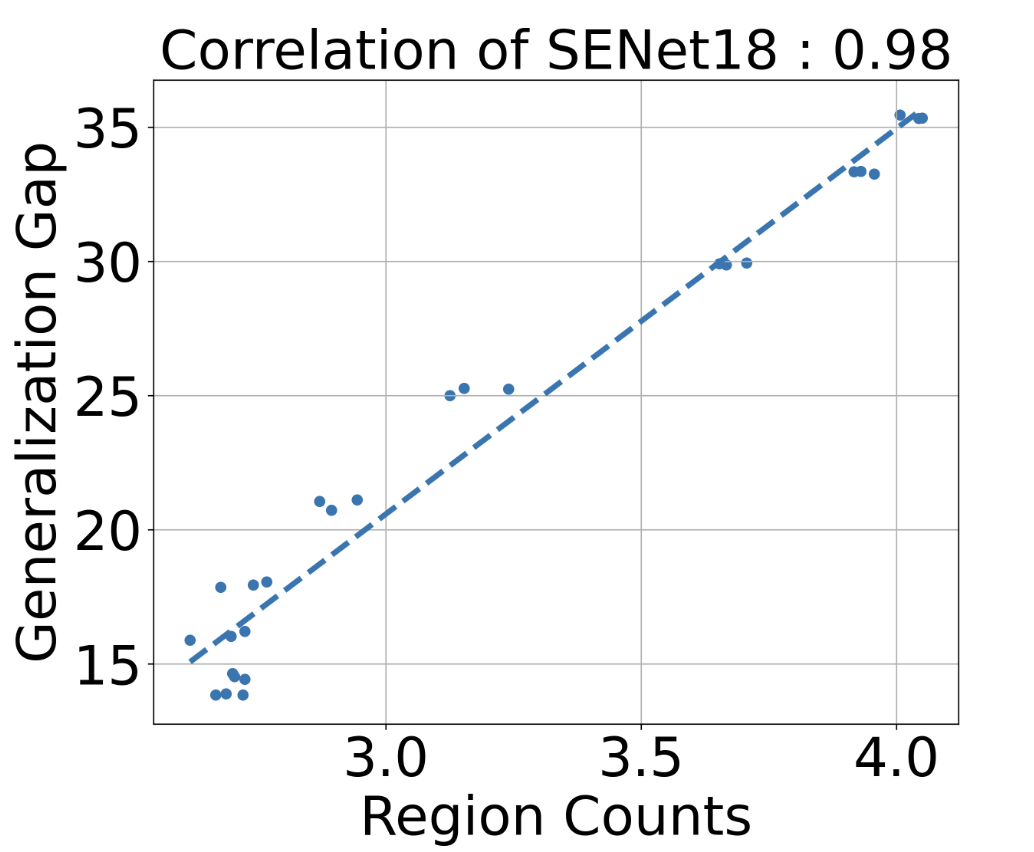}
        \centering
    \end{minipage}
    \hfill
    \begin{minipage}[b]{0.321\linewidth}
        \includegraphics[width=\linewidth]{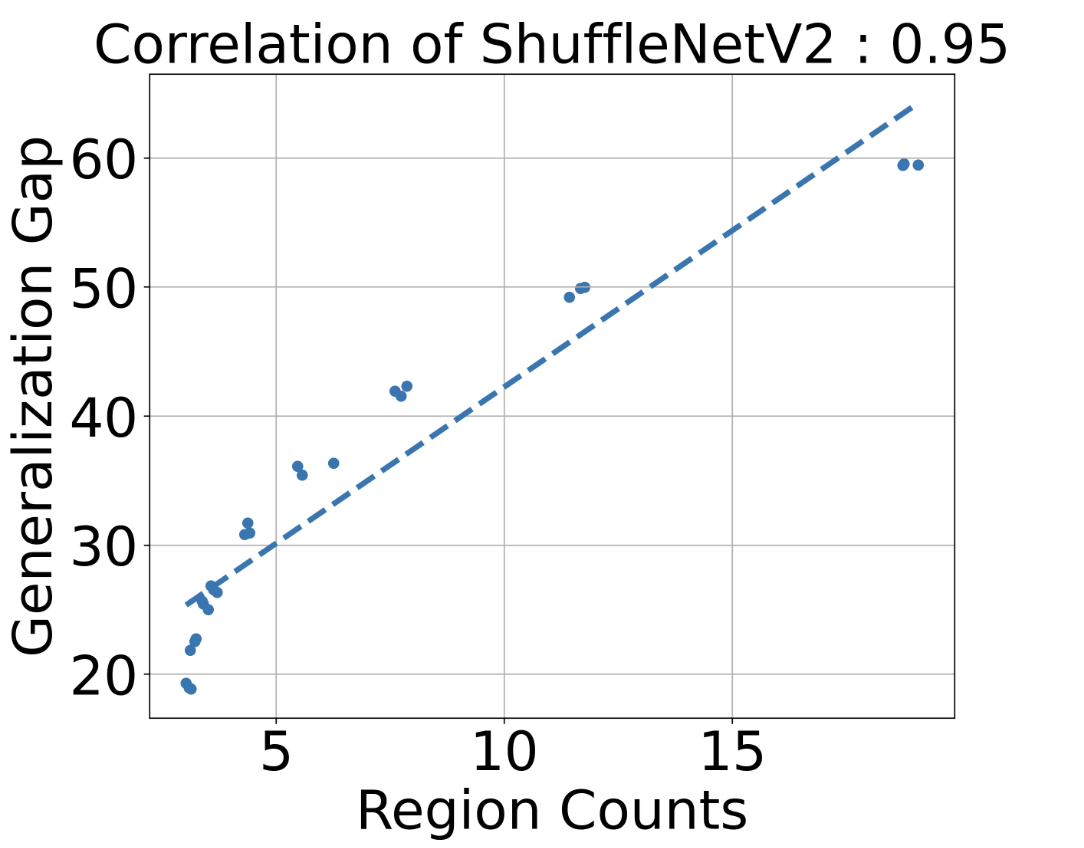}
        \centering
    \end{minipage}
    \hfill
    \begin{minipage}[b]{0.317\linewidth}
        \includegraphics[width=\linewidth]{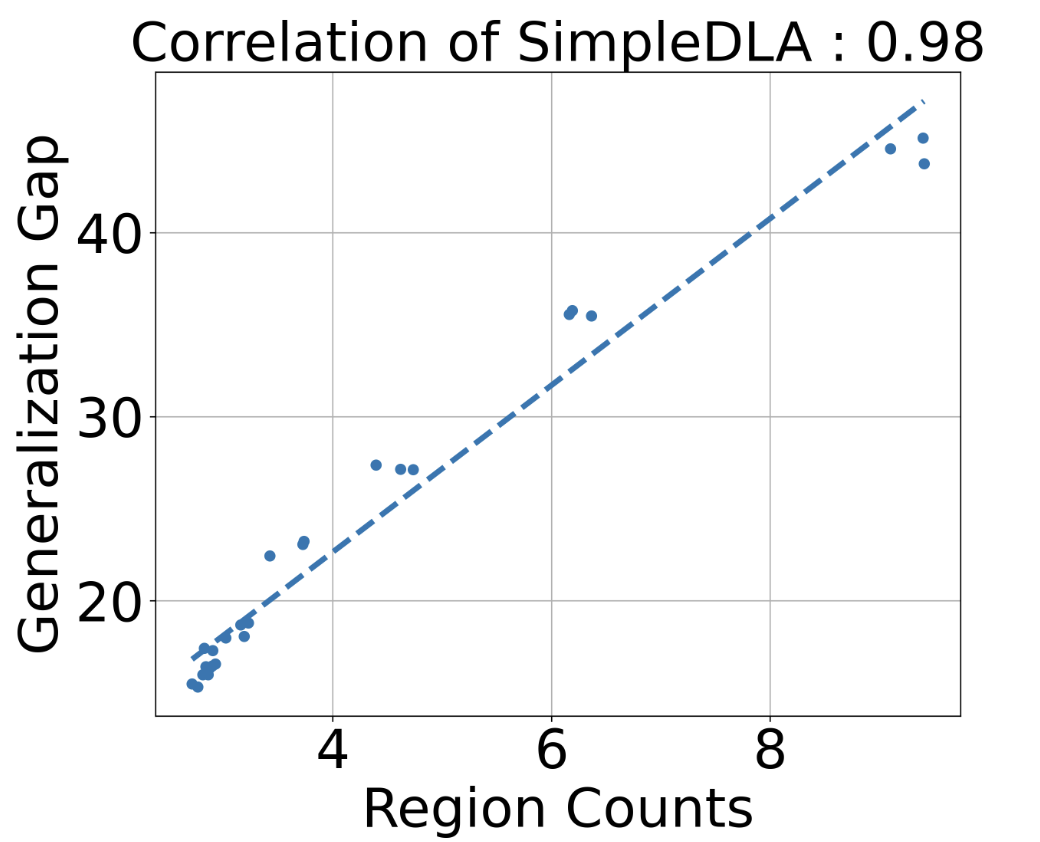}
        \centering
    \end{minipage}
    \caption{The correlation plot of all networks between average regions and test accuracy for CIFAR-10 dataset with optimizer SGD. From the graph we know that the correlations are all very high for different non-linear networks. Different structures of neural networks incur different scope of the average regions. }
    \label{fig:all1}
\end{figure}


\section{How to Calculate the Number of Regions}
\label{app:count}

In this section, we study different methods to calculate the region count, and discuss their impact on the results. Since it is impossible in practice to calculate the predictions of an infinite number of data points on the hyperplane, we select grid points from the hyperplane to calculate the region count. 

Assume we have divided a region of the hyperplane into several equidistant small squares. We can use an algorithm similar to breadth-first search to calculate the number of connected components within these small squares, thereby determining the number of regions. Here, we use a 2-dimensional hyperplane as an example (the 1-dimensional case can be considered a degenerate version of this algorithm). The algorithm for calculating the number of regions in this setup is given in Algorithm~\ref{algo:calculate region}.

\begin{algorithm}
\caption{Calculate the Number of Region}
\label{algo:calculate region}
\begin{algorithmic}[1]
\renewcommand\algorithmicrequire{\textbf{Input:}}
\REQUIRE{Prediction Matrix $P$ with dimension$(w,h)$.}
\renewcommand\algorithmicensure{\textbf{Output:}}
\ENSURE{Number of connected regions $N$.}

\STATE{Initialize a mark matrix $M$ to a zero matrix, with the same dimensions as $P$}
\STATE{Initialize count of connected regions $N \gets 0$}

\FOR{$i \gets 0$ \textbf{to} $w-1$}
\FOR{$j \gets 0$ \textbf{to} $h-1$}
\IF{$M[i][j]$ is already marked}
\STATE \textbf{continue}
\ENDIF
\STATE{Mark position $(i,j)$ in $M$ as visited}
\STATE{Perform Breadth-First Search (BFS) starting from position $(i,j)$}
\STATE{In BFS, enqueue all neighboring points that have the same value as $(i,j)$ in $P$ and mark them as visited in $M$}
\STATE{Continue BFS until the queue is empty}
\STATE{Increment count of connected regions $N \gets N + 1$}
\ENDFOR
\ENDFOR
\STATE \textbf{return} $N$
\end{algorithmic}
\end{algorithm}

Therefore, it is necessary to determine the granularity of splits for the plane. We experimented with different setups of splitting parameters, and the results averaged by 100 independent trials are presented in Table~\ref{tab:split1d}. From the results, using 200 grid points in the 1D case and 30x30 grid points in the 2D case is an optimal choice. Splitting the plane into fewer points results in an inadequate approximation of regions, while increasing the number of points does not significantly enhance accuracy but incurs greater computational costs. Therefore, in our experiments, we split the plane into 200 grid points for the 1D case and 30x30 grid points for the 2D case.

\begin{table}[h]
\small
\centering
\captionsetup{skip=5pt}
\caption{The mean value of region counts with different splitting numbers in 1d (left) and 2d (right) planes. }
\label{tab:splitting}
\begin{minipage}{0.45\linewidth}
\centering
\label{tab:split1d}
\begin{tabular}{cc}
\toprule  
Splitting Numbers & Region Counts \\
\midrule
50  & 2.74\\ 
100 & 2.76\\
200 & 2.78\\
300 & 2.78\\
500 & 2.78\\
\bottomrule
\end{tabular}
\end{minipage}%
\hfill
\begin{minipage}{0.45\linewidth}
\centering
\label{tab:split2d}
\begin{tabular}{cc}
\toprule  
Splitting Numbers & Region Counts \\
\midrule
10$\times$10   & 2.76\\ 
20$\times$20   & 2.76\\
30$\times$30   & 2.78\\
40$\times$40   & 2.78\\
50$\times$50   & 2.78\\
\bottomrule
\end{tabular}
\end{minipage}
\end{table}

Subsequently, we study the number of random samples in calculating the average number of regions. We experiment with different numbers of hyperplanes, and the results are presented in Table~\ref{tab:average}. 
From the results, we know that using 100 samples to calculate the average provides a reliable answer with relatively low computational costs. Therefore, in the experiments we randomly generate 100 lines or planes and calculate the average number of regions.

\begin{table}[h]
\small
\centering
\captionsetup{skip=5pt}
\caption{The mean value of region counts with different number of random samples. }
\label{tab:average}
\begin{tabular}{ccc}
\toprule  
 Number of Samples & Region Counts   \\
\midrule
 10  &2.24  \\ 
 50 & 2.56  \\ 
 100 & 2.78\\
 300 & 2.80\\
 500 & 2.79\\
\bottomrule
\end{tabular}
\end{table}

In our paper we use the convex hull of two points $\{\alpha x_1 + (1-\alpha x_2)\}$ to calculate the region counts in 1D case with $\alpha \in [0,1]$. We also conduct ablation studies with varied coordinate ranges $\alpha$. We train ResNet18 on CIFAR10 using hyperparameter in Table 1 in our manuscript, where we vary the range of $\alpha$ and analyze the correlation. The results are shown in Table~\ref{tab:scope}. These studies confirm that expanding the range does not influence the strong correlation between region counts and test accuracy.

\begin{table}[H]
\small
\centering
\captionsetup{skip=5pt}
\caption{\small{The impact of interpolation range on region counts.}}
\label{tab:scope}
\begin{tabular}{ccc}
\toprule  
  The range of $\alpha$&   Region Counts &    Correlation \\
\midrule
 $[0,1]$  &  3.56 & 0.98  \\ 
 $[-1,2]$  & 4.47 & 0.96 \\ 
 $[-2,3]$ & 5.86 & 0.92\\ 
 $[-3,4] $ & 6.39&0.93\\ 
\bottomrule
\end{tabular}
\end{table}

\newpage
\newpage
\section{More ablation studies}
\label{app:ablation}
\paragraph{Gradient Optimizers.}
We calculate the region count of models trained by different optimizers, including SGD, Adam, and Adagrad. The correlation between region count and the generalization gap is consistent for them, as detailed in Table~\ref{tab:optimizer variance}.
\begin{table}[h]
\small
\centering
\captionsetup{skip=5pt}
\caption{\small{The impact of optimizers on the correlation between region counts and generalization gap.}}
\label{tab:optimizer variance}
\begin{tabular}{ccccc}
\toprule  
  \diagbox{Network}{Optimizer}&  \emph{SGD} & \emph{Adam} & \emph{Adagrad} \\
\midrule
 ResNet18 & 0.98 & 0.92 &0.96\\ 
 ResNet34 & 0.98& 0.92 &0.91\\ 
 VGG19 & 0.94 & 0.92& 0.87\\ 
 MobileNet & 0.95 & 0.95& 0.99\\  
 SENet18 & 0.98 & 0.78& 0.91\\ 
 ShuffleNetV2 & 0.95& 0.83& 0.99\\ 
 EfficientNetB0 & 0.98&0.97 & 0.99\\ 
 RegNetX\_200MF & 0.98&0.95 &0.99 \\ 
 SimpleDLA & 0.98&0.95 & 0.88\\ 
\bottomrule
\end{tabular}
\end{table}

\paragraph{Hyperplane Generation Methods.}
We explore the effects of different methods for generating the hyperplane in the input space. In the main experiments, we generate a 1-dimensional plane using random pairs of samples from the training set and calculated the number of distinct regions between them. In this section, we apply various techniques for plane generation: selecting two data points from the test set, choosing one data point from the training set and extending it in a random direction by a fixed length. We calculate the number of regions for each of these setups. The results in Table~\ref{tab:counting variance} are consistent across different hyperplane computational approaches.

\begin{table}[h]
\small
\centering
\captionsetup{skip=5pt}
\caption{\small{The impact of calculation methods on the correlation between region counts and generalization gap.}}
\label{tab:counting variance}
\begin{tabular}{cccc}
\toprule  
  \diagbox{Network}{Counting}&  \emph{Test} & \emph{Train} & \emph{Random}  \\
\midrule
 ResNet18 & 0.98 & 0.98 & 0.98  \\ 
 ResNet34 & 0.98 & 0.96&0.94  \\ 
 VGG19 & 0.94& 0.89 & 0.78 \\ 
 MobileNet & 0.95 & 0.94 &  0.88\\ 
 SENet18 & 0.98 &0.96 & 0.99 \\ 
 ShuffleNetV2 & 0.95& 0.95&0.92  \\ 
 EfficientNetB0 & 0.98& 0.98& 0.92\\ 
 RegNetX\_200MF & 0.98& 0.97& 0.92  \\ 
 SimpleDLA &0.98 & 0.97 & 0.96\\ 
\bottomrule
\end{tabular}
\end{table}

\paragraph{Data Augmentations.}

Here, we present the correlation curves after applying mixup and random crop. The results show that the correlation between region count and generalization gap remains high after using these data augmentation techniques.

\begin{figure*}[h]
  \centering
    \includegraphics[scale = 0.33]{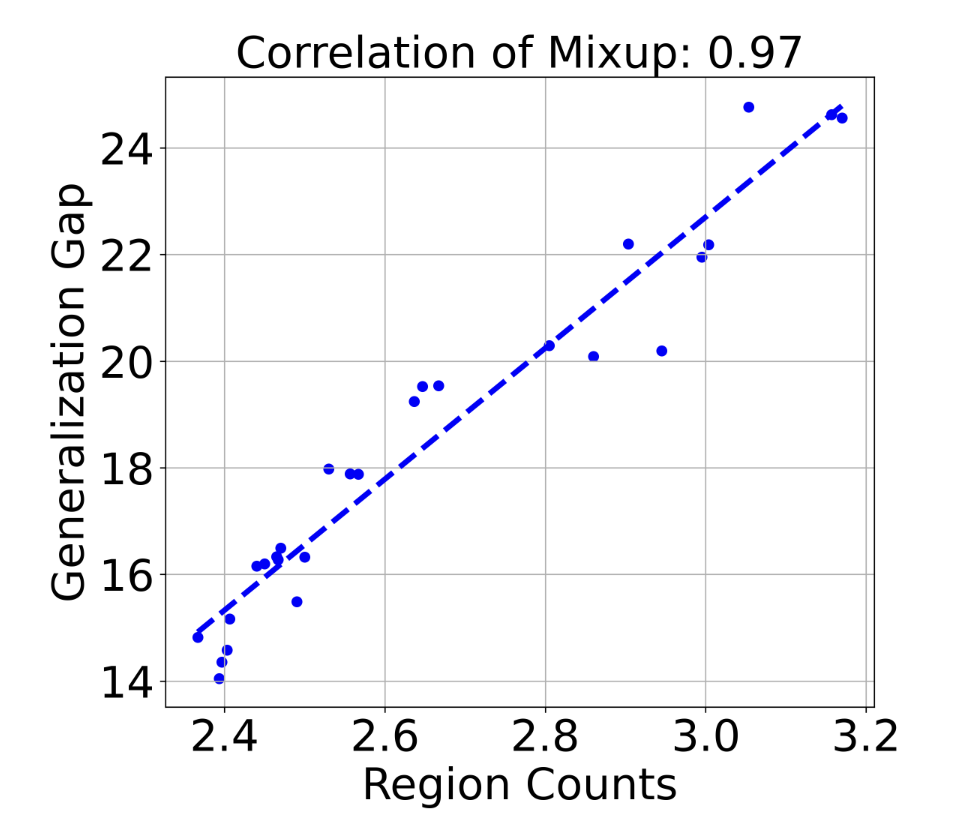}\hspace{1.8cm}
    \includegraphics[scale = 0.395]{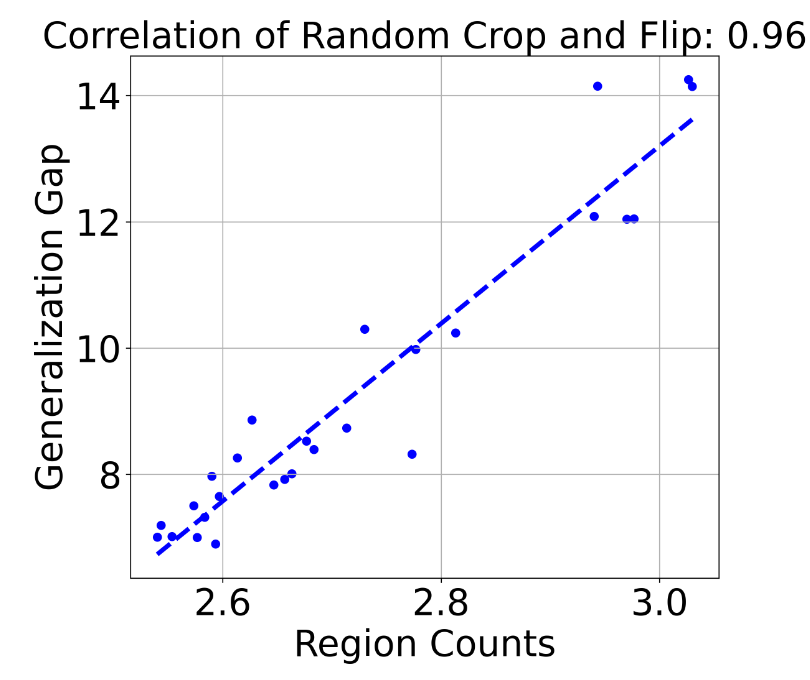}
    \caption{\textbf{The correlation graph after using data augmentation techniques.} We train Resnet18 on the CIFAR-10 dataset, varying the hyperparameters in Table~\ref{tab:hyperparameter}. Our findings reveal that the correlation between region count and generalization gap remains high after using these techniques.}
\end{figure*}

\section{Proof}
\label{app:proof}
This section contains the proof of the theorem in this paper.





We first prove two lemmas. 
\lemregion*
\begin{proof}
    If $R(x_a, x_b, W)\le 2$, then the equation naturally holds. Next we consider $R(x_a, x_b, W)>2$.
    From the definition of region count, one can find $R:=R(x_a, x_b, W)$ points on the line segment between $x_a$ and $x_b$, such that the neural network gives different predictions. Denote these points as $\Tilde{x}_1, \cdots, \Tilde{x}_R$. We have $$f_W(\Tilde{x}_i)f_W(\Tilde{x}_{i+1})<0, 0\le i\le R-1.$$

    Consider $\Tilde{x}_i, \Tilde{x}_{i+1}, \Tilde{x}_{i+2}$. Since the neural network gives alternating predictions on these three points, it is nonlinear and has activation sign changes on the line segment connecting them. Therefore, we can find a $1\le n(i)\le p$, such that $(w_{n(i)}^\top \Tilde{x}_i)(w_{n(i)}^\top \Tilde{x}_{i+2})< 0$.
    
    We prove it by contradiction. If for all $1\le i\le p$, such that $(w_i^\top \Tilde{x}_i)(w_i^\top \Tilde{x}_{i+2})\ge  0$. Suppose $\Tilde{x}_{i+1}=\lambda \Tilde{x}_{i} + (1-\lambda)\Tilde{x}_{i+2}$, then we have 
    \begin{align*} f_W(\Tilde{x}_{i+1}) &= \sum_{i=1}^p a_i \sigma(w_i^\top x_{i+1})\\
    &= \sum_{i=1}^p a_i [\lambda\sigma(w_i^\top x_{i}) + (1-\lambda)\sigma(w_i^\top x_{i+2})] \\
    &= \lambda f_W(\Tilde{x}_{i}) + (1-\lambda)f_W(\Tilde{x}_{i+2}) \,.
    \end{align*}
    Therefore $f_W(\Tilde{x}_{i+1})$ has the same sign of $f_W(\Tilde{x}_{i})$ and $ f_W(\Tilde{x}_{i+2})$, contradict with the condition that they have alternative signs. So we can find a $1\le n(i)\le p$, such that $(w_{n(i)}^\top \Tilde{x}_i)(w_{n(i)}^\top \Tilde{x}_{i+2})< 0$.

    Since $w_{n(i)}^\top x$ is linear in $x$, this implies that $$(w_{n(i)}^\top x_a)(w_{n(i)}^\top x_b)< 0.$$
We also prove it by contradiction. If they have the same sign then the convex combination of them have the same sign so $(w_{n(i)}^\top \Tilde{x}_i)(w_{n(i)}^\top \Tilde{x}_{i+2})\ge 0$.

    We have the following two observations about $n(i)$. Firstly, we can choose an $n(i)$ such that $a_{n(i)}w_{n(i)}^\top \Tilde{x}_{i+2}$ and $f_W(\Tilde{x}_{i+2})$ have the same sign, since there exists at least one such neuron that contributes to the sign change of $f_W$.
    This implies that $n(i)\neq n(i+1)$, since $f_W(\Tilde{x}_{i})$ have alternating signs. Secondly, since $w_{n(i)}^\top x$ is a linear function in $x$, it can only changes sign for at most one time. This implies that $n(i)\neq n(j)$ if $j-i\ge 2$. Putting them together, we know that $n(i)\neq n(j)$ for $i\neq j$.

    Recall that for each $1\le i\le R-2$, we have $(w_{n(i)}^\top x_a)(w_{n(i)}^\top x_b)< 0$. Therefore, there exists $R-2$ neurons that are  activated for either $x_a$ or $x_b$. This gives $N(x_a, W)+N(x_b, W)\ge R-2$, which completes the proof. 
    
\end{proof}

\lemsmooth*

\begin{proof}
    The Hessian of $l(W, x, y)$ can be expressed as 
    $$ \nabla^2_W l(W,x,y)=
    \begin{bmatrix}
    v_1v_1^\top & \cdots  & v_1v_p^\top \\      
    \vdots & & \vdots \\       
    v_pv_1^\top & \cdots & v_pv_p^\top 
    \end{bmatrix},
    $$
    where $v_i = a_i \sigma'(w_i^\top x)x$. Suppose $V = [v_1^\top, \cdots , v_p^\top]$. As the nonzero eigenvalue of $V^\top V$ and $V V^\top$ is the same, this implies that 
    $$\lambda_{\max}(\nabla^2_W l(W,x,y)) = \lambda_{\max}(V V^\top) =\sum_{i=1}^p \|v_i\|_2^2=\sum_{i=1}^p \sigma' (w_i^
    \top x) \|x\|^2\ge \sum_{i=1}^p \sigma' (w_i^
    \top x) r^2.$$

    From the definition of $\nabla^2_W L(W)$ and the positive definiteness of Hessian matrices, we know that 
    \begin{align*}
\lambda_{\max}\left(\nabla^2_W L(W)\right) &= \frac{1}{N} \lambda_{\max}\left(\sum_{i=1}^N \nabla^2_W l(W,x_i,y_i)
        \right)\ge \frac{1}{N^2} \sum_{i=1}^N \lambda_{\max}\left(\nabla^2_W l(W,x_i,y_i) 
        \right).
    \end{align*}
    Plug in the previous calculation and use the definition of $N(x)$, we have 
    \begin{align*}
        \lambda_{\max}\left(\nabla^2_W L(W)\right) \ge \frac{r^2}{N^2} \sum_{i=1}^N\sum_{j=1}^p \sigma'(w_i^\top x_j) = \frac{r^2}{N^2} \sum_{i=1}^N N(x_i, W).
    \end{align*}
\end{proof}

\begin{proof}[Proof of Theorem~\ref{thm:2}]
Theorem~\ref{thm:2} is a direct consequence of Assumption~\ref{asm:eos} and the following two lemmas.
    \begin{align*}
    \E_{X, X'}[R(X, X', W_t)] &= \frac{1}{N^2}\sum_{i=1}^N \sum_{j=1}^N R(x_1, x_2, W_t)
    \\&\le \frac{1}{N^2}\sum_{i=1}^N \sum_{j=1}^N (N(x_i, W_t)+N(x_j, W_t)+2)
    \\&= \frac{2}{N} \sum_{i=1}^N (N(x_i, W_t)+1)
    \\& \le \frac{2N}{r^2}\lambda_{\max}(\nabla_W^2 L(W_t))+2
    \\&= O\left(\frac{N}{r^2\eta}\right)
\end{align*}
\end{proof}
\end{document}